\documentclass[a4paper,USenglish,cleveref, autoref, thm-restate]{lipics-v2021}
\usepackage{algorithm}
\usepackage{algpseudocode}
\usepackage{bm}
\usepackage{mathabx}
\usepackage{multirow}

\newtheorem{prop}{Proposition}
\newtheorem{exmp}{Example}
\newtheorem{defn}{Definition}

\newcommand{\indep}{\perp \!\!\! \perp}
\newcommand{\an}{\text{An}}
\newcommand{\cl}{\text{Cl}}
\newcommand{\pa}{\text{Pa}}

\newcommand{\G}{\mathcal{G}}
\newcommand{\X}{\mathbf{X}}

\title{Analyzing Complex Systems with Cascades Using Continuous-Time Bayesian
  Networks}

\author{Alessandro Bregoli}
  {Department of Informatics, Systems and Communication, University of
   Milano-Bicocca, Italy}{a.bregoli1@campus.unimib.it}
  {https://orcid.org/0000-0002-1743-4441}{}
\author{Karin Rathsman}
  {European Spallation Source ERIC, Lund, Sweden \and
   \url{https://europeanspallationsource.se} }{karin.rathsman@ess.eu}
  {https://orcid.org/0009-0005-0715-8905}{}
\author{Marco Scutari}
  {Istituto Dalle Molle di Studi sull'Intelligenza Artificiale (IDSIA),
   Lugano, Switzerland \and \url{https://www.bnlearn.com}}
  {scutari@bnlearn.com}{https://orcid.org/0000-0002-2151-7266}{}
\author{Fabio Stella}
  {Department of Informatics, Systems and Communication, University of
   Milano-Bicocca, Italy \and \url{https://www.unimib.it/fabio-antonio-stella}}
  {fabio.stella@unimib.it}{https://orcid.org/0000-0002-1394-0507}{}
\author{S\o{}ren Wengel Mogensen}
  {Department of Automatic Control, Lund University, Sweden \and
   \url{https://soerenwengel.github.io/} }{soren.wengel_mogensen@control.lth.se}
  {https://orcid.org/0000-0002-6652-155X}
  {The work of SWM was funded by a DFF-International Postdoctoral Grant
   (0164-00023B) from Independent Research Fund Denmark. SWM is a member of the
   ELLIIT Strategic Research Area at Lund University.}

\authorrunning{A. Bregoli, K. Rathsman, M. Scutari, F. Stella and S. W. Mogensen}

\Copyright{Alessandro Bregoli, Karin Rathsman, Marco Scutari, Fabio Stella, and
  S{\o}ren Wengel Mogensen}



\EventEditors{Alexander Artikis, Florian Bruse, and Luke Hunsberger}
\EventNoEds{3}
\EventLongTitle{30th International Symposium on Temporal Representation and Reasoning (TIME 2023)}
\EventShortTitle{TIME 2023}
\EventAcronym{TIME}
\EventYear{2023}
\EventDate{September 25--26, 2023}
\EventLocation{NCSR Demokritos, Athens, Greece}
\EventLogo{}
\SeriesVolume{278}
\ArticleNo{8}

\begin{document}

\maketitle

\begin{abstract}
  Interacting systems of events may exhibit \emph{cascading} behavior where
  events tend to be temporally clustered. While the cascades themselves may be
  obvious from the data, it is important to understand which states of the
  system trigger them. For this purpose, we propose a modeling framework based
  on \emph{continuous-time Bayesian networks} (CTBNs) to analyze cascading
  behavior in complex systems. This framework allows us to describe how events
  propagate through the system and to identify likely \emph{sentry states}, that
  is, system states that may lead to imminent cascading behavior. Moreover,
  CTBNs have a simple graphical representation and provide interpretable
  outputs, both of which are important when communicating with domain experts.
  We also develop new methods for knowledge extraction from CTBNs and we apply
  the proposed methodology to a data set of alarms in a large industrial system.
\end{abstract}

\section{Introduction}

\begin{figure}[b!]
  \centering
  \begin{subfigure}[c]{0.26\textwidth}
    \centering
    \includegraphics[width=\textwidth]{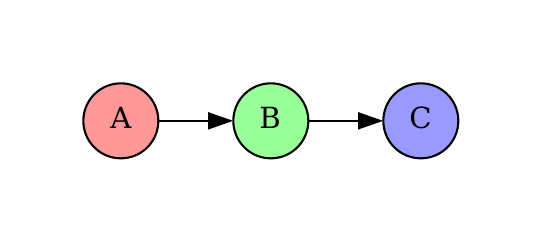}
    \caption{CTBN graph $\G$}
    \label{fig:chain3:Network}
  \end{subfigure}
  \hfill
  \begin{subfigure}[c]{0.45\textwidth}
    \centering
    \includegraphics[width=\textwidth]{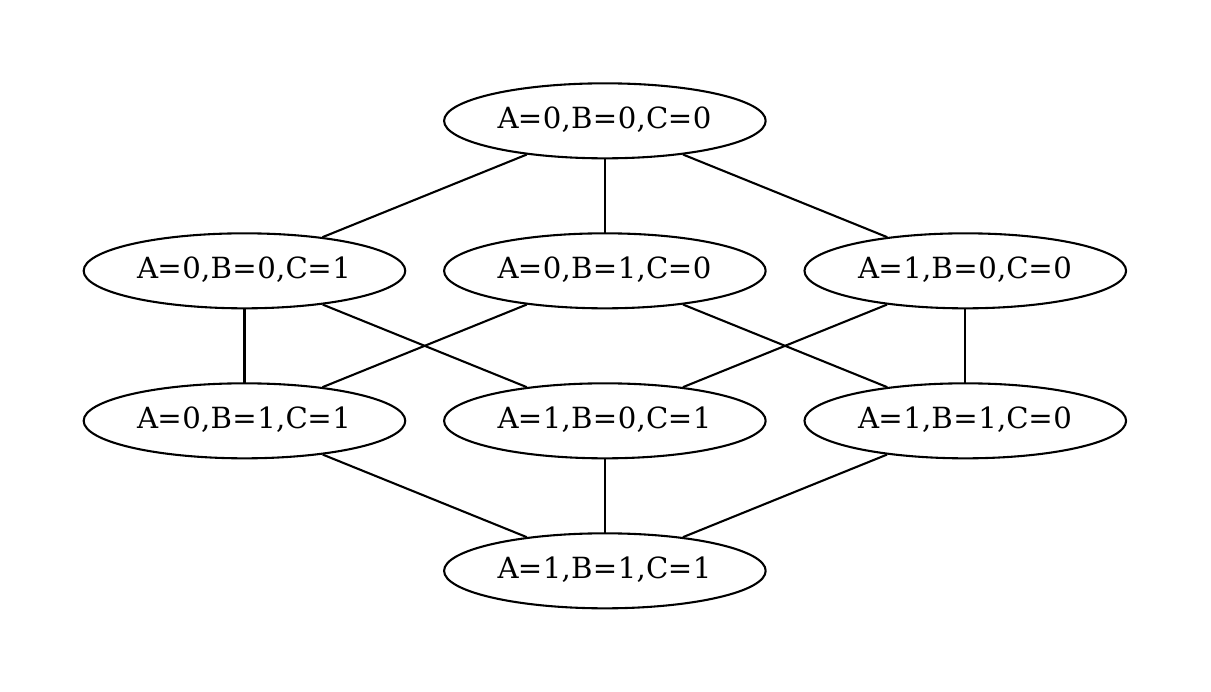}
    \caption{State space graph $\G_s$}
    \label{fig:chain3:state_space}
  \end{subfigure}
  \hfill
  \begin{subfigure}[c]{0.26\textwidth}
    \centering
    \includegraphics[width=\textwidth]{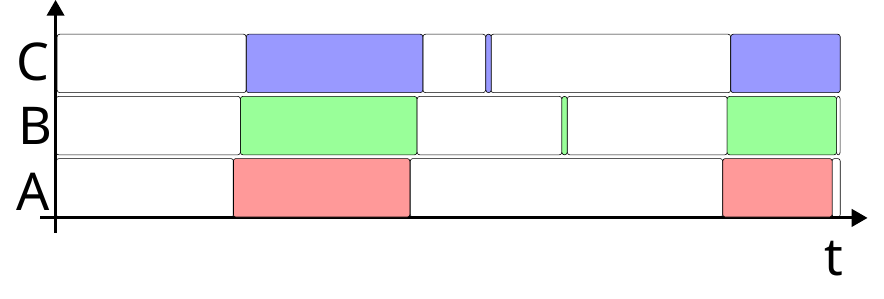}
    \caption{Trajectory}
    \label{fig:chain3:trajectory}
  \end{subfigure}
  \caption{ \emph{a}) Graph $\G$ of a CTBN consisting of three nodes ($A$, $B$,
    $C$). \emph{b}) State space graph $\G_s$ of the CTBN, i.e., nodes represent
    states of the CTBN and two states are adjacent if a transition from one to
    the other, and vice versa, is possible. \emph{c}) A trajectory for the CTBN.
    A white segment indicates that the corresponding process is in state 0
    (off), while a colored segment indicates that the corresponding process is
    in state 1 (on).}
  \label{fig:chain3}
\end{figure}

Many real-world phenomena can be modeled as interacting sequences of events of
different types. This includes social networks where user activity influences
the activity of other users \cite{cencetti2021temporal}. In healthcare, patient
history may be modeled as a sequence of events \cite{weiss2013forest}. In this
paper, we focus on an industrial application in which the events are alarm
signals of a complex engineered system. As an illustration, consider Figure
\ref{fig:chain3}. Three different alarms ($A$, $B$, and $C$) monitor a process
each within an industrial system. These processes may, for instance, represent
measured temperatures or pressures. An alarm transitions to \emph{on} when its
the process it monitors leaves a prespecified range of values and transitions to
\emph{off} when the process is again within the range corresponding to `normal
operation'. In Figure \ref{fig:chain3:trajectory}, the colored segments
correspond to the alarm being \emph{on}. When an alarm changes state, we say
that an \emph{event} occurs.

In particular, we are interested in systems that exhibit a \emph{cascading}
behavior in the sense that events are strongly clustered in time. Therefore, we
should use a model class which is capable of expressing cascades. However,
important information may also be contained in non-cascading parts of the
observed data. This is explained by the fact that our goal is twofold: We wish
to understand which states of the system lead to cascades, and we also wish to
understand how different components of the system interact. These goals are
connected as understanding the inner workings of the system will also help us
understand the cascading behavior. We can achieve both by using
\emph{continuous-time Bayesian networks} (CTBNs) \cite{nodelman2007continuous},
a class of parsimonious Markov processes with a discrete state space. Moreover,
CTBNs are equipped with a graphical representation of how different components
interact. In our application, data is sampled at a very high frequency and using
a continuous-time modeling framework makes for a conceptually and
computationally simple approach. In CTBNs, we define the concept of a
\emph{sentry state} which is a state that may lead to an imminent cascade. In an
industrial setting, identifying such states may give operators an early warning,
which in turn facilitates the mitigation of underlying issues before an actual
alarm cascade occurs. Moreover, sentry states can be used to apply state-based
alarm techniques \cite{hollifield2011alarm}, which otherwise require a known
alarm structure.

In applications to complex systems, analysts often need to communicate findings
to domain experts. CTBNs also allow easy communication using their graphical
representation which is crucial for their applicability to real-world problems.
We describe a useful interpretation of the graph of a CTBN and provide some new
results in this direction. We apply the methods developed in this paper to a
challenging data set from the European Spallation Source (ESS), a research
facility in Lund, Sweden. This data set describes how alarms propagate through a
subsystem of the neutron source at ESS.

The paper is structured as follows. We start with a description of the ESS data
as this motivates the following developments (Section \ref{appendix:Dataset}).
In Section \ref{sec:related}, we describe related work and compare CTBNs with
other approaches. Section \ref{sec:models} describes continuous-time Bayesian
networks. Sections \ref{sec:sentry} and \ref{sec:graphicalinfo} contain the main
theoretical contributions of the paper: Section \ref{sec:sentry} describes the
concept of a sentry state while Section \ref{sec:graphicalinfo} explains how the
graphical representation of a CTBN may assist interpretation and communication.
Section \ref{sec:num} presents numerical experiments, including a description of
the ESS data analysis. A discussion concludes the main paper and the appendices
contain auxiliary results.

\section{Data Set}
\label{appendix:Dataset}

The European Spallation Source ERIC (ESS) is a large research facility which is
being built in Lund, Sweden. Its main components include a linear proton
accelerator, a tungsten target, and a collection of neutron instruments
\cite{essabout}. It comprises a large number of systems, including an integrated
control system \cite{EPICS}. The facility has a goal of 95\% availability and
state-of-the-art alarm handling may contribute to reaching this goal.

Operators of large facilities are often facing large quantities of data in real
time and good tools may helpsystem understanding and support decision making.
Operators rely on alarm systems to warn them about unexpected behavior. However,
alarm problems are common \cite{hollifield2011alarm}. One example is that of
\emph{cascading alarms}. In large facilities, different alarms monitor different
processes and when an issue occurs this may result in a large number of alarms
that occur within a short time frame due to the interconnectedness of the
different processes. Operators will often find it difficult to respond to such
cascades as hundreds or thousands of alarms may sound, making it difficult to
identify the underlying issue.

The alarm system has two purposes. One, it should help operators foresee and
mitigate fault situations. Two, it should help operators understand a fault
situation. In this paper, we illustrate how the methods we propose can help
achieve these goals using data from the accelerator cryogenics plant at ESS,
which has been in operation since 2019.

\begin{exmp}[Simplified ESS alarm network]
  In this example, we will look at a simplified version of the alarm system at
  ESS. Each of the alarm processes P1, T1, P2, T2, P3, T3, S1, S2, S3, and S4 in
  Figure  \ref{fig:systemgraph} (left) monitors a physical process, e.g., a
  temperature or a pressure. At each time point, each alarm process is either 1
  (alarm is \emph{on}) or 0 (alarm is \emph{off}). An edge, $\rightarrow$, in
  the graph implies a dependence in transition rates, e.g., the rate with which
  P1 changes its state depends only on the current states of P1, T1, and S4.
  Alarm cascades occur when a certain state triggers a fast progression of alarm
  onsets. In this formalism, this is modeled by the dependence of transition
  rates on the current state: therefore, cascades tend to unfold along the
  directed edges of the graph.

  Large engineered systems can often be divided into subsystems (in Figure
  \ref{fig:systemgraph}, Systems 1, 2, 3, and 4). In alarm networks, a group of
  alarms may monitor processes that are known to be correlated as they are
  measured from the same subsystem. Moreover, many systems of a realistic size
  comprise so many alarms or variables that processes must be grouped into
  subsystems to enable a system-level understanding of their dependencies. In
  Section \ref{sec:graphicalinfo}, we show that graphical representations using
  these subsystems (e.g., Figure \ref{fig:systemgraph} (right)), instead of the
  processes themselves, are also useful and have a clear interpretation.

  \begin{figure}[ht!]
    \begin{subfigure}[c]{.7\textwidth}
       \includegraphics[width=\textwidth]{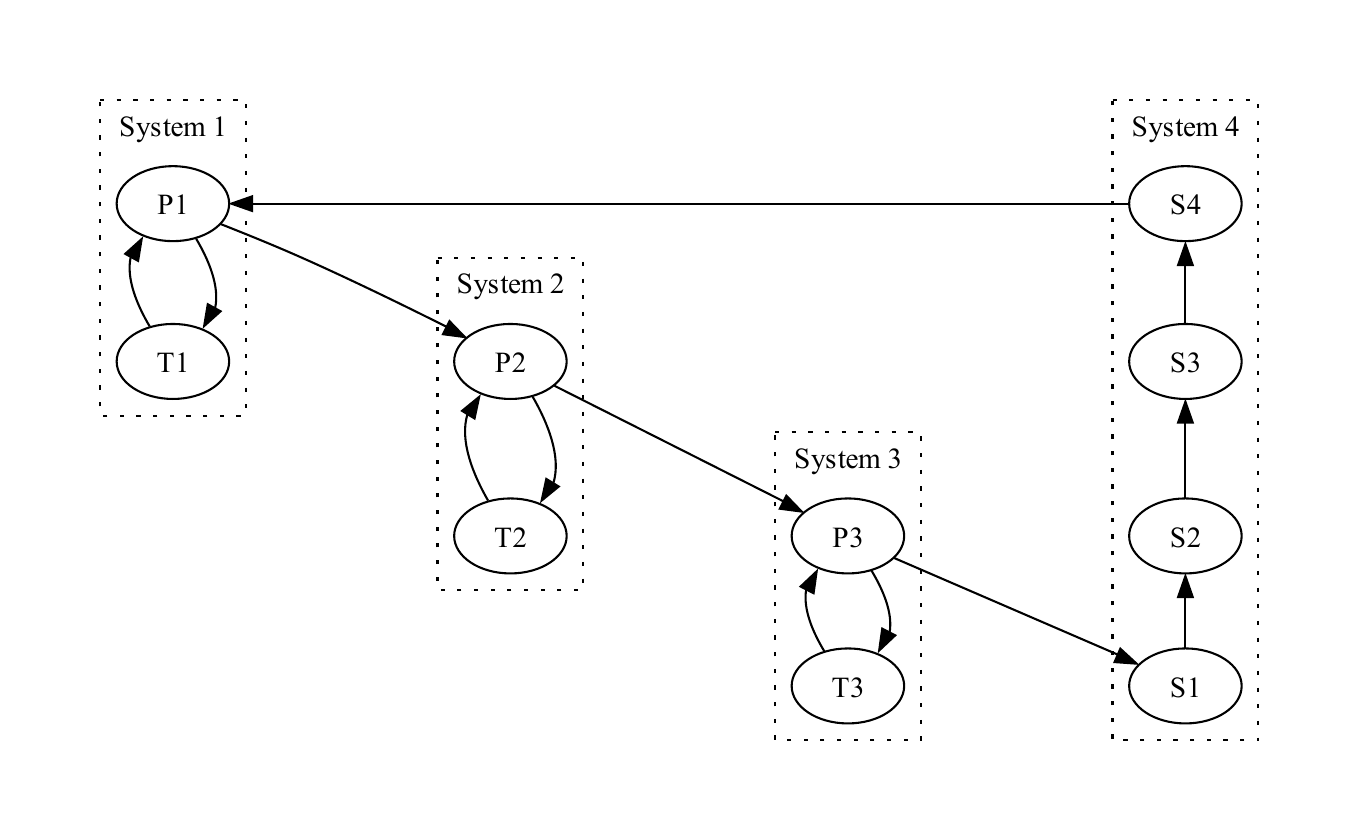}
    \end{subfigure}
    \hfill\vrule\hfill
    \begin{subfigure}[c]{.25\textwidth}
      \includegraphics[width=\textwidth]{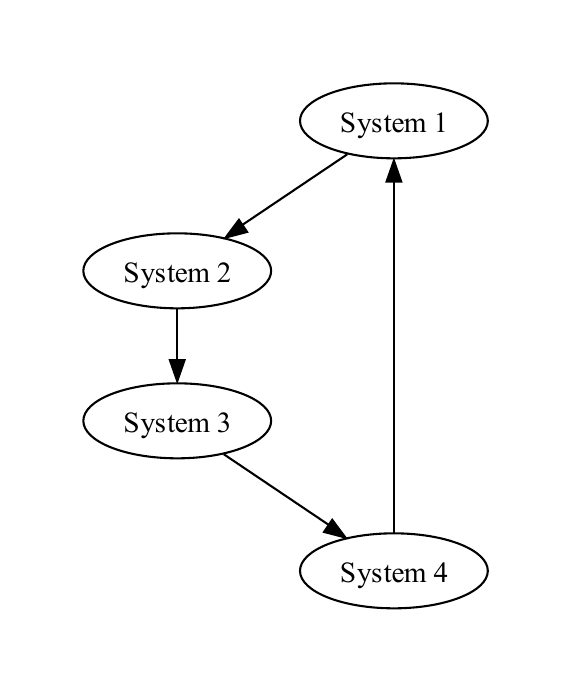}
    \end{subfigure}
    \caption{Graphs from Example \ref{exmp:system1}. Left: Graph such that each
      node represents an alarm process. Edges, $\rightarrow$, represent sparsity
      in transition rate dependence. The transition rate of each process, $A$,
      only depends on its own state and the states of its \emph{parent}
      processes, i.e., processes $B$ such that the edge $B\rightarrow A$ is in
      the graph. Right: Graph representing transition rate dependencies between
      entire subsystems rather than between individual alarm processes. This
      graph is computed from the larger graph on the left and can be given a
      mathematical interpretation (Section \ref{sec:graphicalinfo}).}
    \label{fig:systemgraph}
  \end{figure}

\label{exmp:system1}
\end{exmp}

\section{Related Work}
\label{sec:related}

Our work has connections to several major directions in the literature of
dynamical models. We will focus on methods that are relevant for modeling
cascades of events.

Continuous-time Markov processes with a discrete state space are often used as
models of cascading failures in complex networks such as power grids
\cite{rahnamay2014stochastic, nakarmi2020markov} and as models of burst behavior
in biological cells \cite{ball2002multivariate}. They are also used in chemistry
\cite{anderson2011continuous}, reliability modeling \cite{iyer2009markovian},
queueing theory \cite{medhi2002stochastic}, and genetics
\cite{bergman2018inference}. Graphs may be used as representations of networks
and several methods use graphs to represent cascading structures
\cite{netrapalli2012learning, nakarmi2020interaction}. Some of these methods use
simple models of contagion and study influential nodes in a graph
\cite{lappas2010finding}. Among these we find the \emph{independent cascade
model} and the \emph{linear threshold model} \cite{kempe2003maximizing,
chakrabarti2008epidemic}. These models are targeted at applications in which
cascades or `epidemics' are the salient feature.

\emph{Change-point detection} methods aim to recover the time points at which
distributional changes occur. There is a large literature on change-point
detection in various classes of stochastic processes and application fields such
as neuroscience  \cite{ritov2002detection}, DNA sequence segmentation
\cite{braun1998statistical}, speech recognition \cite{rybach2009audio}, and
climate change \cite{reeves2007review}. \cite{aminikhanghahi2017survey} provides
a survey on change-point methods in discrete-time models.
\cite{truong2020selective} also provides a survey and discusses subsampling of
continuous-time processes. \cite{wang2022sequential} provides a method for
change-point detection in a class of multivariate point processes. There are
subtle, but important, differences between the task at hand and change-point
detection. The alarm network that motivates our work is thought to operate in
the same mode throughout the observation period. Moreover, the detection of
cascades is trivial in this application as they are evident from simply
visualizing the data. Instead, we focus on identifying the states that are
likely to lead to a cascade. In addition, our modeling approach should allow for
a qualitative understanding of the interactions between system components.

As a result, we would like to explicitly model how events propagate through the
system. A traditional approach is to use \emph{point processes}
\cite{daley2003introduction}. In our setting, we can think of a point process as
a sequence of pairs $(t_1,e_1), (t_2,e_2), (t_3,e_3),\ldots$ such that $(t_i)$
is an increasing sequence of time points and $(e_i)$ denotes the type of event.
Point process models have been used for modeling cascades of failures
\cite{lee2016point}. \cite{rambaldi2018detection} describes self-exciting point
processes that model cascading behavior using explicit exogenous influences on
the system. These methods can in principle also be applied in the setting of
this paper. The CTBN-based method proposed in this paper models the state of the
system directly and this facilitates the notion of \emph{sentry states} which is
central to this paper. Moreover, as the alarm status takes values in a discrete
set, the CTBN provides a natural representation of the alarm data.

While the above describes general approaches to the modeling of dynamical
systems with cascading behavior, there are also more specialized methods for
handling or analyzing alarm cascades in large industrial systems. We now
summarize the connections to this work; see also \cite{alinezhad2022review},
which is a recent review of methods for alarm cascade problems (in that paper
known as alarm \emph{floods}). In short, so-called \emph{knowledge-based}
methods use expert knowledge of the cause-effect relations in the system to find
root causes and explain alarm cascades. Among these are \emph{multilevel flow
modeling} \cite{lind2013overview} and \emph{signed directed graphs}
\cite{wan2013statistical}. In contrast, there is also a large number of
\emph{data-based} methods. Data-based methods can be subdivided into classes of
methods depending on their purpose. Some approaches aim to classify the fault
type from an input sequence of alarms or to simply reduce the number of alarms,
e.g., using data mining, clustering, or machine learning methods such as
\emph{artificial neural networks} \cite{yang2012improved, rodrigo2016causal,
arunthavanathan2022autonomous}. This goal is somewhat different from ours and
our method is more closely related to methods using \emph{probabilistic
graphical models}, in particular dynamic Bayesian networks \cite{hu2015fault}.
These methods produce an easily interpretable output represented by a graph. Our
contribution in relation to this prior work is twofold. 1) We introduce the CTBN
framework as a natural further development of alarm modeling. This allows
continuous-time modeling which is useful for our data as it is collected using a
very high sampling rate. Discretization of time will therefore either lead to a
prohibitively large number of time intervals or to a loss of temporal
information if using fewer, longer time intervals. 2) We define the novel
concept of a \emph{sentry state} which identifies a set of states that are of
interest when analyzing cascading behavior. The sentry state concept may also be
used in other alarm propagation models and is not specific to CTBNs.

\section{Models}
\label{sec:models}

Continuous-time Bayesian networks (CTBNs; \cite{nodelman2002continuous}) are a
class of continuous-time Markov processes (CTMPs) with a factored state space
and a certain sparsity in how transition rates depend on the current state. This
sparsity can be represented by a directed graph. In this sense, they are similar
to classical Bayesian networks \cite{pearl1988probabilistic} but their directed
graphs are allowed to contain cycles. CTBNs have proved to be both effective and
efficient representations of discrete-state continuous-time dynamical systems
\cite{AVPMZS16,ijcai2018p804,LHLS2019}. We first define a CTMP.

\subsection{Continuous-Time Markov Processes}

A \emph{continuous-time Markov process} (CTMP) \cite{shelton2014tutorial} is a
continuous-time stochastic process $\X = \{X(t) : t\in \left[0, \infty\right)\}$
which satisfies the following Markov property:
\begin{equation}
  X(t_1) \indep X(t_3) | X(t_2), \hspace{0.2cm} \forall \hspace{0.1cm} t_1 < t_2 < t_3,
\end{equation}
where $\cdot \indep \cdot \mid \cdot$ denotes conditional independence. The
state of the process $\X$ changes in continuous-time and takes values in the
domain $S$ which we assume to be a finite set. In our application, each state
$s\in S$ can be represented by an $n$-vector with binary entries indicating
whether each alarm is \emph{on} or \emph{off}. A CTMP can be parameterized by
the \emph{initial distribution} $P_0$ and the \emph{intensity matrix} $Q_\X$.
The initial distribution $P_0$ is any distribution on the state space. The
intensity matrix $Q_\X$ models the evolution of the stochastic process $\X$.
Each row of $Q_\X$ sums to $0$ and models two distinct processes:
\begin{enumerate}
  \item The time when $\X$ abandons the current state $x$, which follows an
    \emph{exponential distribution} with parameter $q_x \in \mathbb{R}^+$.
  \item The state to which $\X$ transitions when abandoning the state $x$. This
    follows a \emph{multinomial distribution} with parameters $\theta_{xy} =
    \frac{q_{xy}}{q_x}$, $x,y \in S$, $x\neq y$.
\end{enumerate}
An instance of the intensity matrix $Q_X$, when $S$ has three states, is as follows
\begin{equation}
  Q_\X = \begin{bmatrix}
    -q_{x_1} & q_{x_1 x_2} & q_{x_1 x_3} \\
    q_{x_2 x_1} & -q_{x_2} & q_{x_2 x_3} \\
    q_{x_3 x_1} & q_{x_3 x_2} & -q_{x_3} \\
  \end{bmatrix} \;\;\;\; q_{x_i} > 0, q_{x_{i}{x_j}} \geq 0 \; \forall \, i,j.
  \label{eq:QMatrix}
\end{equation}

We say that a realization of a CTMP, $\sigma$, is a \emph{trajectory}. This is a
right-continuous, piecewise constant function of time. It can be represented as
a sequence of time-indexed events,
\begin{align}
  &\sigma = \{\langle t_0, X(t_0)\rangle,
             \langle t_1, X(t_1)\rangle, ...,
             \langle t_I, X(t_I)\rangle\},& &t_0 < t_1 < \dots <t_I.
\label{eq:trajectory}
\end{align}

\subsection{Continuous-Time Bayesian Networks}

General CTMPs do not assume any sparsity. CTBNs impose structure on a CTMP by
assuming a factored state space $S = \{S_1 \times S_2 \times \dots \times
S_L\}\;$ such that $ X(t) = (X_1(t), \ldots, X_L(t))\in S$  where each
$S_j$,~$j=1, ..., L$, represents the domain of a distinct component of the
process.\footnote{A CTBN is specialization of a CTMP. It is possible to
reformulate a CTBN as a CTMP by applying the so-called amalgamation procedure
\cite{shelton2014tutorial}.} In the alarm data, $S_j = \{0,1\}$ for each $j$ and
$X_j(t)$ indicates if the $j$'th alarm is \emph{on} or \emph{off} at time $t$.
The structure imposed by the CTBN is useful when interpreting a learned model.
In essence, the CTBN framework allows us to learn which components of the system
act independently, or conditionally independently, and this can be communicated
to experts.

A CTBN is a tuple ${\cal N}=\langle P_0,  \X, {\cal G}, {\bf Q}_X\rangle$ where
$\X = \{\X_1,\ldots,\X_L\}$ is a set of stochastic processes. A CTBN is
specified by:
\begin{itemize}
  \item An initial probability distribution $P_0$ on the factored state space
    $S$.
  \item A continuous-time transition model, specified as:
  \begin{itemize}
    \item a directed (possibly cyclic) graph $\G$ with node set $\textbf{X}$;
    \item a set of conditional intensity matrices $\mathbf{Q}_{\X_j|\pa(\X_j)}$
      for each process $\X_j\in \X$.
  \end{itemize}
\end{itemize}

Given the graph $\G$, each node/process $\X_j$ has a \emph{parent set}
$\pa(\X_j)$ consisting of all nodes/processes, $\X_i$, with an edge directed
from $\X_i$ to  $\X_j$ in $\G$, $\X_i \rightarrow \X_j$. A \textit{conditional
intensity matrix} (CIM) $\mathbf{Q}_{\X_j|\pa(\X_j)}$ consists of a set of
intensity matrices, one for each possible configuration of the states of the
parent set $\pa(\X_j)$ of the node/process $\X_j$, that is, one for each element
of $\bigtimes_{\X_i\in \pa(\mathbf{X_j})} S_i$. The CIM describes how the
transition intensity of process $j$ depends on the state of the system. However,
it does not necessarily depend on the state of every other process, but only on
the states of the processes that are parents of $j$.

In a CTBN only one process can transition at any given time. This assumption is
reasonable for the alarm data as it is sampled at a high rate. When we apply
this model to the alarm data, the interpretation is straightforward. The CIM of
$\X_j$ describes how likely this alarm is to change its state (from
\emph{on} to \emph{off}, or from \emph{off} to \emph{on}), and this only depends
on the current state of the alarms in its parent set.

\begin{exmp}
  Assume we observe three alarm processes ($A$, $B$, and $C$) each monitoring a
  measured process and that we represent this by a stochastic process $\X$ that
  takes values in $\{0,1\} \times \{0,1\} \times \{0,1\}$ indicating the status
  of each of the three alarms. If $\X$ is a general CTMP, then the transition
  rates of each alarm process may depend on the entire state of the system. On
  the other hand, if $\X$ is a CTBN represented by the graph $\G$ in Figure
  \ref{fig:chain3:Network} there is a certain sparsity in the way the transition
  rates depend on the current state. It follows directly that the transition
  rate of alarm process $C$ only depends on the states of processes $B$ and $C$.
  Similarly, the transition rate of alarm process $B$ only depends on the states
  of processes $A$ and $B$ while the transition rate of process $A$ only depends
  on its own state. When we learn a CTBN from data, we can therefore use its
  graph as a qualitative summary of the interconnections of the alarms. An edge
  from $A$ to $B$ in this graph means that a change in the state of $A$ may
  change the transition rates of $B$ and therefore cascades are expected to
  occur along the directed edges of the graph.

  Another type of graphical representation is also useful: Figure
  \ref{fig:chain3:state_space} shows the associated factored state space graph
  $\G_s$. In this graph, each node represents a state (in contrast, in $\G$ each
  node represents a process/alarm) and edges represent the possible transitions
  between states. We will say that $\G$ is the graph of the CTBN and refer to
  $\G_s$ as the state space graph.
\end{exmp}

As we will see in the numerical experiments, CTBNs are capable of producing
`cascading' behavior. However, they also model the non-cascading behavior: This
is important for our application because we would also like to use the
information contained outside periods with cascades. The CTBN framework has the
following advantages that are critical to our application: 1) It exploits the
factorization of the multivariate alarm process  to make it feasible to learn a
CTBN model from data. 2) It takes into account the duration of an event as well
as its occurrence. Moreover, it uses both the cascading and non-cascading data.
3) It has a graphical representation which is easy to interpret, thus
facilitating communication.

\subsection{Reward Function}

A \emph{reward function} is a function that maps the states of one or more
processes onto a real number. We use a reward function to compute the
discounted, expected number of transitions (that is, alarms changing their
states) when starting the process in some initial state, $x$. A reward function
consists of two quantities,
\begin{itemize}
  \item $\mathcal{R}(x): S \rightarrow \mathbb{R}$, the \emph{instantaneous
    reward} of state $X=x$, and
  \item $\mathcal{C}(x, x'): S \times S\rightarrow \mathbb{R}$, the \emph{lump
    sum reward} when $X$ transitions from state $x$ to state $x'$.
\end{itemize}
We use the lump sum reward which is an indicator of transitions,
\begin{equation}
  \mathcal{C}(x, x') = 1, \ \ \ \text{ for all } x,x' \in S,
\label{eq:lump_sum_reward}
\end{equation}
and we let the instantaneous reward be zero. We will use the
\emph{infinite-horizon expected discounted reward} \cite{guo2009continuous},
\begin{equation}
  V_\alpha(x) = \mathbb{E}_x \left[ \sum_{i=0}^\infty
    e^{-\alpha t_{i+1}}C(X(t_i), X(t_{i+1})) + \int _{t_i}^{t_{i+1}}
    e^{-\alpha t}\mathcal{R}(X(t_i)) dt \right] : t_i < t_{i+1}
\label{eq:infinite_horizon}
\end{equation}
where $\alpha > 0$ is referred to as the \emph{discounting factor},
$\mathbb{E}_x(\cdot)$ is the expectation when conditioning on $ X(0) = x $, and
the $t_i$'s are the transition times.  We use $\mathcal{R}(x) = 0$ for all $x$
and therefore
\begin{equation}
  V_\alpha(x) = {\mathbb{E}_x \left[ \sum_{i=0}^\infty
    e^{-\alpha t_{i+1}} C(X(t_i), X(t_{i+1})) \right]} =
    \mathbb{E} \left[ \sum_{i=0}^\infty e^{-\alpha t_{i+1}} \right]:
    t_i < t_{i+1}.
\label{eq:infinite_horizon_2}
\end{equation}
This simply counts the number of transitions including a discounting factor.
Clearly, other reward functions can be chosen to analyze other or more
specialized types of behavior. If, for instance, we are only interested in a
subset of transitions we can modify the lump sum reward accordingly. A value of
the parameter $\alpha$ can be chosen using, e.g., prior information on the
length of typical cascades. This concludes the introduction of the modeling
framework and the following sections describe the contributions of this paper.

\section{Sentry State}
\label{sec:sentry}

Given a CTBN ${\cal N}=\langle P_0, \X, \G, \mathbf{Q}_X\rangle$, we
are interested in understanding its cascading behavior. We are in particular
interested in identifying what we will call \emph{sentry states}. A sentry state
is a state which may trigger a \emph{ripple effect}, that is, a sequence of fast
transitions.

\begin{exmp}
  An example of a ripple effect can be found in the CTBN in Figure
  \ref{fig:chain3}a which consists of three processes $\X_A$, $\X_B$ and $\X_C$,
  forming a \emph{chain} $A\rightarrow B \rightarrow C$, and each taking values
  in $\{0, 1\}$, $S_A=S_B=S_C=\{0, 1\}$. The trajectory in Figure
  \ref{fig:chain3:trajectory} shows that every time $A$ transitions, $B$ and $C$
  quickly transition as well. In other words, when $A$ changes its state, a
  ripple effect occurs such that $B$ and $C$ also change states to match the
  state of their parent. The starting point of these cascades of events is a
  \emph{sentry state} as defined in this section.

  \begin{figure}[ht!]
    \centering
    \includegraphics[width=\textwidth]{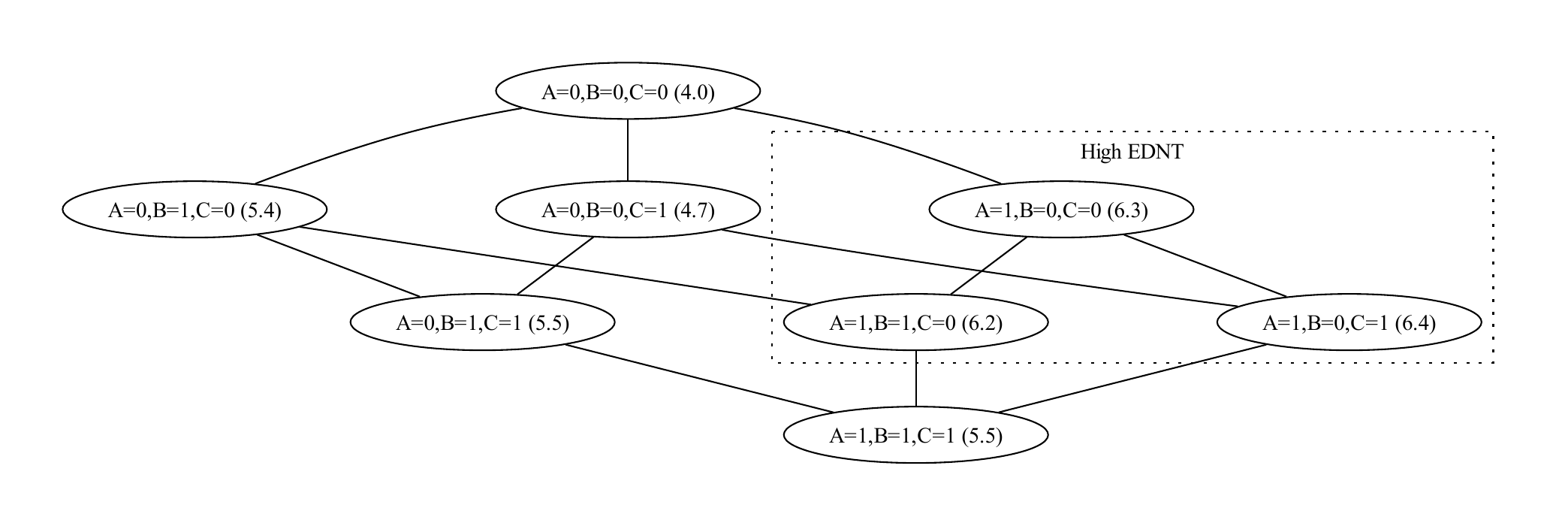}
    \caption{Visualization of Example \ref{exmp:chain3}: The state space graph
      (same graph as in Figure \ref{fig:chain3:state_space}, but presented
      slightly differently) is annotated with EDNTs for each node (in
      parentheses) and a subset of nodes with relatively high EDNT values is
      highlighted. Sentry states are high EDNT states to which transitions from
      low EDNT states are possible and such a transition increases the risk of
      an imminent cascade. Some states with high EDNT tend to occur in the
      middle of a cascade which motivates using the REDNT metric. In this
      example, transitioning from $(A=0,B=0,C=0)$ to $(A=1,B=0,C=0)$ may trigger
      an alarm in $B$ which in turn may trigger an alarm in $C$, resulting in a
      cascade (note that only transitions along the edges in the state space
      graph are possible).}
  \label{fig:chain3withEDNT}
  \end{figure}

\label{exmp:chain3}
\end{exmp}

It is important to stress that we are interested in states that \emph{start} a
cascade of events. Intuitively, this means that we are assuming the existence of
at least one state in the state space graph which is directly connected to the
sentry state and which has a much smaller expected number of transitions than
the sentry state itself. We can observe this in the example in Figure
\ref{fig:chain3}c: Before starting the sequence of transitions of processes
$\X_A$, $\X_B$, and $\X_C$ from $0$ to $1$, the CTBN remained for a long time in
state $\{A=0, B=0, C=0\}\in S$. Similarly, before changing the state of all
processes $A$, $B$ and $C$ from $1$ to $0$, the CTBN remained for a long time in
state $\{A=1, B=1, C=1\}\in S$.

\subsection{Sentry State Identification}

In order to identify a sentry state, we need to take a further step from the
heuristic definition of a sentry state we have just given and formalize the
concept. For this purpose, we first compute the expected (discounted) number of
transitions for each state of the CTBN. This can be achieved by using the lump
sum reward in (\ref{eq:lump_sum_reward}) to obtain the \emph{Expected Discounted
Number of Transitions} (EDNT) of each state $x\in S$,

\begin{equation}
  EDNT_\alpha(x) = E \left[ \sum_{i=0}^\infty
    e^{-\alpha t_{i+1}}C(X(t_i), X(t_{i+1}))  \right] : C(x,x') = 1.
\label{eq:EDNT}
\end{equation}

There is no guarantee that a state with high EDNT is often the starting point of
a cascade. States that tend to occur in the middle of a cascade may easily have
a high EDNT if the cascade tends to continue after reaching that state. We are
interested in early detection of cascades and the solution we propose is to take
into account the number of transitions in the neighborhood $\text{Ne}_{\G_s}(x)$
of the state $x\in S$. For this purpose, we define a new quantity called
\emph{Relative Expected Discounted Number of Transitions} (REDNT),
\begin{equation}
  REDNT_\alpha(x) = \max_{x' \in \text{Ne}_{\G_s}(x)}
    \frac{EDNT_\alpha(x)}{EDNT_\alpha(x')}
\label{eq:REDNT}
\end{equation}
where $\alpha$ in (\ref{eq:EDNT}) and (\ref{eq:REDNT}) is the discounting factor
as in (\ref{eq:infinite_horizon}), and the neighborhood in (\ref{eq:REDNT})
refers to the undirected state space graph (see Figure
\ref{fig:chain3:state_space}). The central idea is that a large ratio between
two adjacent states implies that transition from one to the other leads to a
significant change in the expected discounted number of transitions. We will use
REDNT to identify potential sentry states (states with high values of REDNT are
likely sentry states).

One could propose other ways to aggregate EDNT across different states. We focus
on REDNT as defined above in the interest of brevity. We let $\Bar{s}$ denote
the number of alarms that are \emph{on} in the state $s = (s_1,s_2,\ldots,s_L)$,
$\Bar{s} = \sum_{i = 1}^L s_i$. In the alarm data application, we are mostly
interested in sentry states such that $\Bar{s}$ is fairly small. States with
large $\Bar{s}$ may also have large REDNT values; however, these are states that
occur when a cascade is already happening. As we want early detection, we should
focus on sentry states such that $\Bar{s}$ is small.

\subsection{Monte Carlo Algorithm}

We are now left with the problem of estimating the EDNT of each state from which
we can compute the REDNT. We propose a Monte Carlo approach based on Algorithm
\ref{algo:ctbn_sample} from  \cite{nodelman2007continuous}. This sampling
algorithm starts from an initial state $X(0)$ and generates a single trajectory
$\sigma$ ending at time $t_{end}$. After the \textit{initialization} phase the
algorithm enters into a loop. At each iteration, the algorithm samples a time to
transition for each of the variables, identifies the next \textit{transitioning
variable}, generates the \textit{next state}, and \textit{resets the time to
transition} for the transitioned variable and all its children. We combine
Algorithm \ref{algo:ctbn_sample} with (\ref{eq:EDNT}) to compute
\begin{equation}
  \widehat{EDNT}_\alpha(x) = \frac{1}{|\bm\sigma|} \sum_{\sigma \in \bm\sigma}
    \sum _{i=0}^{|\sigma|}e^{-\alpha t_i}C(x(t_i), x(t_{i+1}))
\label{eq:EDNT_est}
\end{equation}
where $|\bm\sigma|$ is the number of trajectories generated by Algorithm
\ref{algo:ctbn_sample} and $|\sigma|$ represents the number of events in the
trajectory $\sigma$.

\begin{algorithm}
  \caption{Forward Sampling for CTBN.}
  \begin{algorithmic}
    \vspace{0.0cm}
    \Procedure{CTBN-Sample}{$X(0)$, $t_{end}$}
      \State $t \gets 0$, $\sigma \gets \{\langle 0, X(0)\rangle\}$
        {\footnotesize\Comment Initialization}
      \Loop
        \ForAll{ $\X_i \in \X$ s.t. $Time(\X_i)$ is undefined}
          {\footnotesize\Comment Time to transition sampling} \State $\Delta t
          \gets$ draw a sample from an exponential with rate
          $q_{\X_i(t)|\pa(\X_i(t))}$
          \State $Time(\X_i) \gets t + \Delta t$
        \EndFor
        \State $j = \arg\min_{\X_i \in \X}\left[ Time(\X_i) \right]$
          {\footnotesize\Comment Transitioning variable}
        \If{$Time(\X_j) > t_{end}$}
          \State \textbf{return} $Tr$
        \EndIf
        \State $x_j \gets $ draw a sample from a multinomial with
          $\theta_{\X_j(t)|\pa(\X_j(t))}$
          {\footnotesize\Comment Next state}
        \State $t \gets Time(\X_j)$
        \State Add $\langle t, X(t) \rangle$ to $\sigma$
        \State Undefine $Time(\X_j)$ and $Time(\X_i)$ $\forall \X_i\in \pa(\X_j)$
          {\footnotesize\Comment Reset the time to transition}
      \EndLoop
      \State \textbf{return} $\sigma$
    \EndProcedure
  \end{algorithmic}
\label{algo:ctbn_sample}
\end{algorithm}

In order to compute $\widehat{EDNT}_{\alpha}$, we need to set the values of the
following hyperparameters:
\begin{enumerate}
  \item $\alpha$, the discounting factor;
  \item $t_{end}$, the ending time for each trajectory;
  \item $|\bm\sigma|$, the number of trajectories to be generated.
\end{enumerate}
Choosing the discounting factor $\alpha$ and the ending time $t_{end}$, we
decide the importance of the \textit{distant future} and appropriate values
depend on the application. On the other hand, $|\bm\sigma|$ controls the
trade-off between the quality of the approximation and its computational cost:
We can choose its value using a stopping-rule approach based on variance as
proposed in \cite{bicher2022review}.

\section{Graphical Information}
\label{sec:graphicalinfo}

This paper proposes the CTBN as a modeling tool for systems with cascades.
However, CTBNs have other useful properties: The interplay between the graph and
the probabilistic model facilitates both communication with subject matter
experts and easy computation of various statistics that summarize the learned
system.

We define the \emph{parent set} of $A$, $\pa(A) = \left(\bigcup_{\X_i
\in A} \pa(\X_i)\right) \setminus A$. Note that $\pa(A) \cap \{A\}=\emptyset$
for all $A$. We define the \emph{closure} of $A$, $\cl(A)= \pa(A)\cup \{A\}$. A
\emph{walk} is a sequence of adjacent edges and a \emph{path} is a sequence of
adjacent edges such that no node is repeated. We say that $\X_i$ is an
\emph{ancestor} of $\X_j$ if there exists a \emph{directed} path from $\X_i$ to
$\X_j$, $\X_i \rightarrow \ldots \rightarrow \X_j$, such that all the edges
point toward $\X_j$. We define $\an(A)$ to be the set of nodes that are in $A$
or are ancestors of a node in $A$. Therefore, $A\subseteq \an(A)$ for all $A$.
We say that the node set $A$ is \emph{ancestral} if $\an(A) = A$, that is, if
$A$ contains all ancestors of every node in $A$. In Figure
\ref{fig:3_nodes_experiment}, the set $\{A,B\}$ is ancestral while the set
$\{B,C\}$ is not ancestral. We let $\G_A$ denote the graph with node set $A$
such that for $\X_i,\X_j\in A$, the edge $\X_i\rightarrow \X_j$ is in $\G_A$ if
$\X_i\rightarrow \X_j$ is in $\G$. We construct the \emph{moral graph}, $\G^m$,
of $\G$ by replacing all edges with \emph{undirected} edges, $-$, and adding an
undirected edge between two nodes if there exists a node of which they are both
parents. In an undirected graph and for disjoint $A$, $B$, and $C$, we say that
$A$ and $B$ are \emph{separated} by $C$ if every path between $A$ and $B$ is
intersected by $C$.

\subsection{Decomposition Properties}

The graph of a CTBN has a clear interpretation, as the transition rate of
$\X_i$ only depends on the current value of $\cl(\X_i)$. The
following results provide another interpretation of the graph in terms of
conditional independence. We let $\overline{\X}_A(t)$ denote the process
$A$ until time point $t$, $\overline{\X}_A(t) = \{{X}_i(s): i\in A,
s\leq t\}$. The next result follows from Proposition 5 in
\cite{didelez2007graphical}.

\begin{prop}
  Let $\G = (\X, E)$ be the graph of a CTBN, and let $A,B,C\subseteq V$ be
  disjoint. If $A$ and $B$ are separated by $C$ in $(G_{\an(A\cup B \cup
  C)})^m$, then $\overline{\X}_A(t) \indep \overline{\X}_B(t) \mid
  \overline{\X}_C(t)$ for all $t$.
\label{prop:indep}
\end{prop}

The above result allows us to decompose the learned system into components $A$
and $B$ that operate independently conditionally on $C$. That is, all dependence
between $A$ and $B$ is `explained' by $C$. However, the graph may not be very
informative to experts if it contains too many nodes. We address this issue now
and further results are in Appendix \ref{app:graph}.

\subsection{Hierarchical Analysis}

The graph of a CTBN may be too large to be easily examined visually. If the
number of components in a system is large experts mostly reason about groups of
components. For instance, in the alarm data, the system is known to comprise
different subsystems, which form a natural partition of the components
$\X$. We show that Proposition \ref{prop:indep} still applies in an aggregated
version of the graph. We define a \emph{graph partition} to formalize this.

\begin{defn}[Graph partition]
  Let $\G = (\X,E)$ be a directed graph and let $D = \{D_1, \ldots, D_m\}$ be a
  partition of $\X$. The \emph{graph partition} of $\G$ induced by $D$,
  $\mathcal{D}$, is the directed graph $(D,E_D)$ with node set $D$ such that
  $D_k \rightarrow D_l$, $k\neq l$, is in $E_D$ if and only if there exists
  $\X_i \in D_k$ and $\X_j \in D_l$ such that $\X_i \rightarrow \X_j$ in $\G$.
\label{def:graphPart}
\end{defn}

In a graph partition, each node, $D_k \in D$, corresponds to a subset of the
node set in the original graph. Underlined symbols, e.g., $\underline{A},$
represent subsets of $D = \{D_1,\ldots,D_m\}$. When $\underline{A}\subseteq  D$,
we let $A$ denote the corresponding nodes in the original graph, $A=
\bigcup_{D_i\in \underline{A}} D_i$. The following is an extension of the
classical separation property to a graph partition and it means that separation
in a graph partition can be translated into separation in the underlying graph.
This is in turn implies a conditional independence in the CTBN.

\begin{prop}
  Let $D$ be a partition of $V$ and let $\mathcal{D}$ be the graph partition
  induced by $D$. Let $\underline{A},\underline{B},\underline{C} \subseteq D$ be
  disjoint. If $\underline{A}$ and $\underline{B}$ are separated by
  $\underline{C}$ in $(\mathcal{D}_{\an(\underline{A}\cup
  \underline{B}\cup\underline{C})})^m$, then $A$ and $B$ are separated by $C$ in
  $(\G_{\an(A \cup B\cup C)})^m$.
\label{prop:indepPart}
\end{prop}

\begin{exmp}[Simplified ESS alarm network]
  We revisit the example in Figure \ref{fig:systemgraph} and denote the graph on
  the left by $\G$. This graph represents a CTBN as introduced in Section
  \ref{sec:models}. $\text{System 1}$, $\text{System 2}$, $\text{System 3}$, and
  $\text{System 4}$ constitute a partition, $D$, of the node set of $\G$ and we
  let $\mathcal{D}$ denote the corresponding graph partition (Figure
  \ref{fig:systemgraph} (right)). If we let $\underline{A} = \{\text{System
  1}\}$, $\underline{B} = \{\text{System 3}\}$, and $\underline{C} =
  \{\text{System 2},\text{System 4}\}$, then $\underline{A}$ and $\underline{B}$
  are separated by $\underline{C}$ in $(\mathcal{D}_{\an(\underline{A}\cup
  \underline{B}\cup\underline{C})})^m$ (in this case $\an(\underline{A}\cup
  \underline{B}\cup\underline{C}) = \{\text{System 1}, \text{System 2},
  \text{System 3}, \text{System 4}\}$ and $(\mathcal{D}_{\an(\underline{A}\cup
  \underline{B}\cup\underline{C})})^m$ is simply the undirected version of
  $\mathcal{D}$ as every node only has a single parent). The set $\underline{A}$
  corresponds to processes $A = \{P1, T1\}$, $\underline{B}$ corresponds to
  processes $\{\text{P3},\text{T3}\}$, and $\underline{C}$ corresponds to
  $\{\text{P2},\text{T2}, \text{S1},\text{S2}, \text{S3},\text{S4}\}$.
  Proposition \ref{prop:indepPart} gives that $A$ and $B$ are separated by $C$
  in $(\G_{\an({A}\cup {B}\cup {C})})^m$. It follows from Proposition
  \ref{prop:indep} that the processes in $A$ and the processes in $B$ are
  independent when conditioning on the processes in $C$. This means that the
  state of System 1 is irrelevant when reasoning about the state of System 3 if
  we already account for Systems 2 and 4. Using this procedure, conditional
  independences can be found using both the original graph and a graph
  partition. Furthermore, in both $\G$ and $\mathcal{D}$, the edges have a
  simple interpretation: The transition rates of the processes corresponding to
  a node only depend on the processes corresponding to the parent nodes.
\end{exmp}

\subsubsection{Condensation}

The graph $\G$ of a CTBN may be cyclic. A possible simplification is to collapse
each cyclic component into a node to form the \emph{condensation} of $\G$ which
is a \emph{directed acyclic graph}. We say that $A\subseteq V$ is a
\emph{strongly connected component} if for every $\X_i \in A$ and every
${\X}_j\in A$ there exists a directed path from $\X_i$ to ${\X}_j$. The strongly
connected components constitute a partition of $V$ and the \emph{condensation}
is the graph partition they induce. The condensation has some properties that do
not hold for general graph partitions (see Appendix \ref{app:graph}).

\begin{defn}[Condensation]
  Let $\G$ be a directed graph and let $D = \{D_1, \ldots, D_m\}$ be its
  strongly connected components. We say that the graph partition of $\G$ induced
  by $D$ is the \emph{condensation} of $\G$.
\end{defn}

\section{Numerical Experiments and Examples}
\label{sec:num}

We now study the performance of the proposed approach. We generate synthetic
data from CTBNs such that the sentry states are known. Data is generated from
different CTBNs (additional experiments are in Appendix \ref{app:synth}). In all
of them,
\begin{itemize}
  \item each process, $\X_j\in \X$, has a binary state space.
  \item the CTBN consists of \emph{slow processes} and \emph{fast processes}.
  \item each process, $\X_j \in \X$, replicates the state of its parent
    processes, $\pa(\X_j)$.
  \item if a process, $\X_j \in \X$, has more than one parent, it stays in state
    0 with high probability if at least one of its parents is in state 0.
\end{itemize}

\paragraph*{Experiment 1}

The first synthetic experiment is based on a CTBN model whose graph $\G$ is a
chain consisting of three nodes $A$, $B$, and $C$ (Figure
\ref{fig:3_nodes_experiment}). The corresponding CIMs for the processes $A$,
$B$, and $C$ are shown in Table \ref{table:3_nodes_cims} in Appendix
\ref{app:synth}. This CTBN describes a structured stochastic process such that
the root process, $A$, changes slowly from the state no-alarm (0) to the state
alarm (1) and vice versa. This can be seen from the CIM corresponding to process
$A$. The CIMs associated with processes \textit{B} and \textit{C} make these two
processes replicate the state of their parent process and this happens at a
faster rate. Therefore, starting from $(0,0,0)$, if process $A$ changes its
state, process $B$ quickly changes its state to match that of its parent $A$.
The same holds true for the process $C$. For this reason, we expect $\{A=1, B=0,
C=0\}$ to be a sentry state because as soon as the process $A$ transitions from
state 0 to state 1, a fast sequence of transitions (a cascade of events) makes
the processes $B$ and $C$  transition from  state 0 to  state 1. This behavior
is shown in Figure \ref{fig:3_nodes_trajectory}. Estimates of the REDNT quantity
are shown in Table \ref{table:3_nodes_results} and they confirm that $\{A=1,
B=0, C=0\}$ is a sentry state.

\begin{figure}[t]
  \begin{subfigure}[c]{0.4\textwidth}
  \centering
    \includegraphics[width=\textwidth]{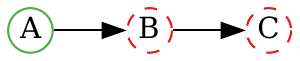}
    \caption{}
    \label{fig:3_nodes_experiment}
  \end{subfigure}
  \begin{subfigure}[c]{0.6\textwidth}
  \centering
    \includegraphics[width=\textwidth]{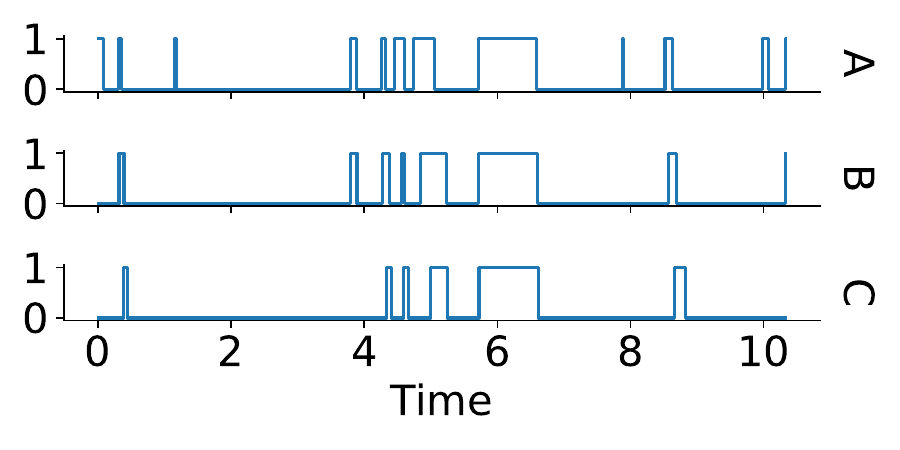}
    \caption{}
    \label{fig:3_nodes_trajectory}
  \end{subfigure}
  \caption{(a) depicts the graph $\G$ of a CTBN. Its CIMs are in Table
    \ref{table:3_nodes_cims} in the appendix. Slow processes are represented by
    solid line nodes, while fast processes are represented by dashed line nodes.
    Colors describe the most likely sentry state, $s = (s_1,s_2,s_3)$, in this
    system: If a node is green, the corresponding alarm is 1 (\emph{on}) in $s$.
    If a node is red, the corresponding alarm is 0 (\emph{off}) in $s$. (b)
    shows an example trajectory from the CTBN represented in Figure
    \ref{fig:3_nodes_experiment}. Each function in the plot represents the
    evolution of one of the three binary processes, $A$, $B$, and $C$. }

\label{fig:chain3sentry}
\end{figure}

\begin{figure}[t]
  \begin{subfigure}[c]{0.4\textwidth}
  \centering
    \includegraphics[width=\textwidth]{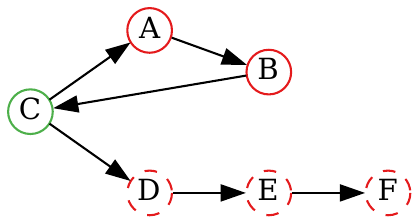}
    \caption{}
    \label{fig:6_nodes_experiment}
  \end{subfigure}
  \begin{subfigure}[c]{0.6\textwidth}
  \centering
    \includegraphics[width=\textwidth]{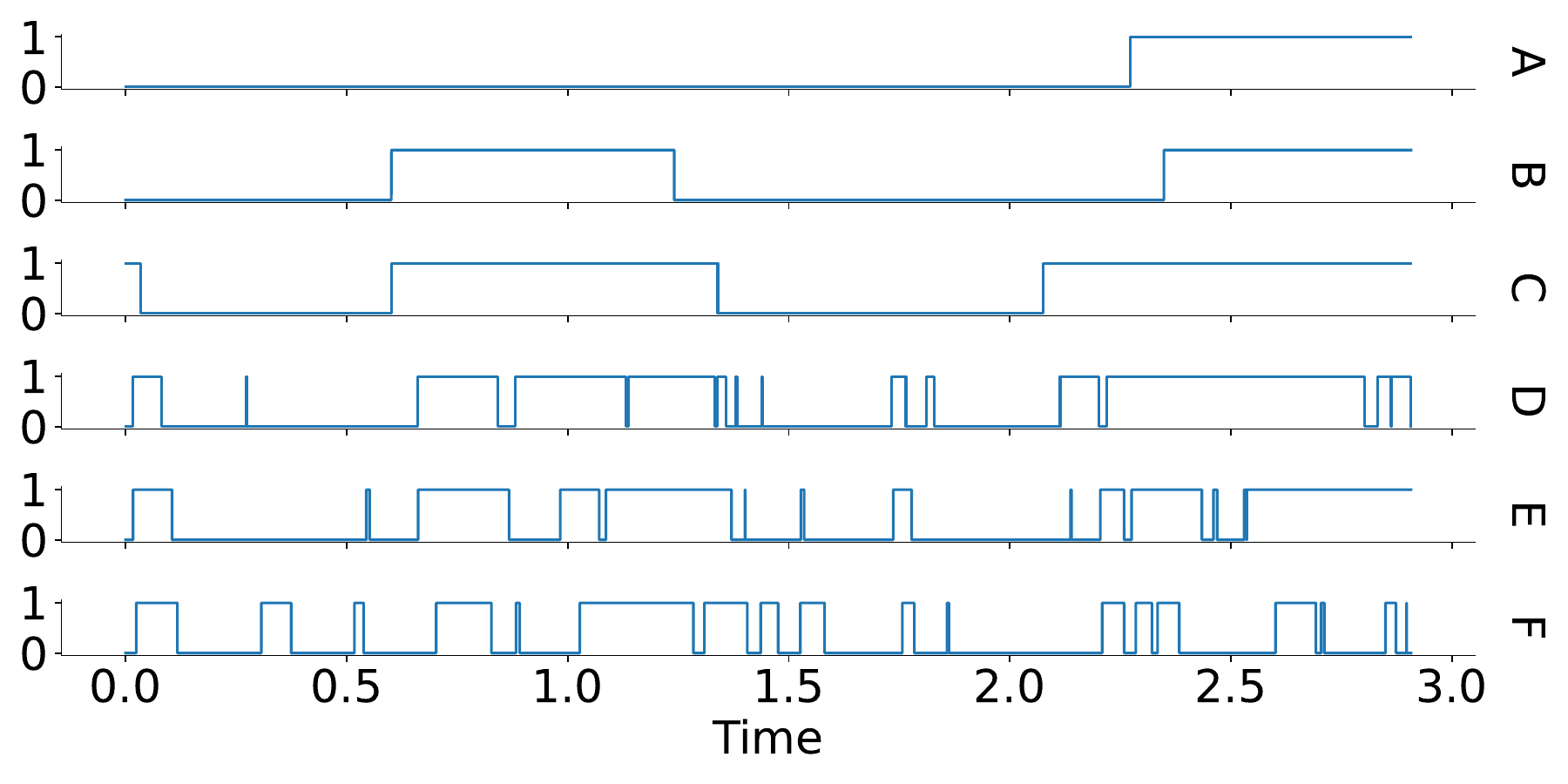}
    \caption{}
    \label{fig:6_nodes_trajectory}
  \end{subfigure}

  \caption{(a) Graph $\G$ of a CTBN model consisting of six processes. The graph
    $G$ contains the cycle ($A$, $B$, $C$) as well as the chain ($D$, $E$, $F$).
    See the caption of Figure \ref{fig:chain3sentry} for an explanation of the
    node colors.}
\end{figure}

\paragraph*{Experiment 2}

The second synthetic experiment is based on the CTBN shown in Figure
\ref{fig:6_nodes_experiment} which consists of a slow cycle ($A$, $B$, $C$) and
a fast chain ($D$, $E$, $F$). In this CTBN, the sentry state is expected to be
$\{A=0, B=0, C=1, D=0, E=0, F=0\}$. Figure \ref{fig:6_nodes_trajectory} shows
that this state triggers a fast sequence of alarms in the chain ($D$, $E$, $F$)
and a slow sequence of alarms in the cycle ($A$, $B$, $C$). Estimates of the
REDNT quantity are shown in Table \ref{table:6_nodes_results} and they confirm
that $\{A=0, B=0, C=1, D=0, E=0, F=0\}$ is a sentry state.

\paragraph*{Comparison}

We compare the REDNT method to the naive approach proposed in Appendix
\ref{subsection:cascade_identification}. In synthetic data it is easier to
identify the two parameters of the naive approach. Each synthetic experiment has
only two transition rates and we can let the parameter $\lambda_{ft}$
\footnote{Threshold between a slow and a fast transition (Appendix
\ref{subsection:cascade_identification}).} be the median elapsed time between
two consecutive events when combining events of all types. The parameter
$\lambda_{mcl}$ \footnote{It determines the minimum number of fast consecutive
events to be considered a cascade  See Appendix
\ref{subsection:cascade_identification}.} can be determined based on the
structure of the network. For instance, in the example in Figure
\ref{fig:3_nodes_experiment} we expect a cascade to have at least two
transitions,
\begin{equation*}
  \{A=1, B=0, C=0\} \rightarrow  \{A=1, B=1, C=0\} \rightarrow \{A=1, B=1, C=1\}.
\end{equation*}
To identify sentry states using the naive approach we should simply identify the
cascades of events and compute the fraction of times that observing a specific
state coincides with the start of a cascade. As already mentioned in Section
\ref{sec:sentry}, we are interested in sentry states with a low number of active
alarms. For this reason, we consider only states such that the number of active
alarms is less than or equal to the size of the largest parent set in the true
graph. The naive approach and the REDNT both produce a list where states are
ordered from the most likely sentry state to the least likely. We compare the
two approaches with the \textit{Jaccard similarity} \cite{murphy1996finley}
using the $K$ most likely sentry states. We tested our approach on 6 different
structures with different numbers of nodes. Results are reported in Figure
\ref{fig:jaccard}. In every experiment, the two methods share at least one state
in their top-two lists. It is important to emphasize that the parameters of the
naive method have been set knowing the length of cascades. Conversely, the REDNT
method does not require this knowledge in order to identify sentry states.

\begin{figure}[t]
  \centering
  \includegraphics[width=\textwidth]{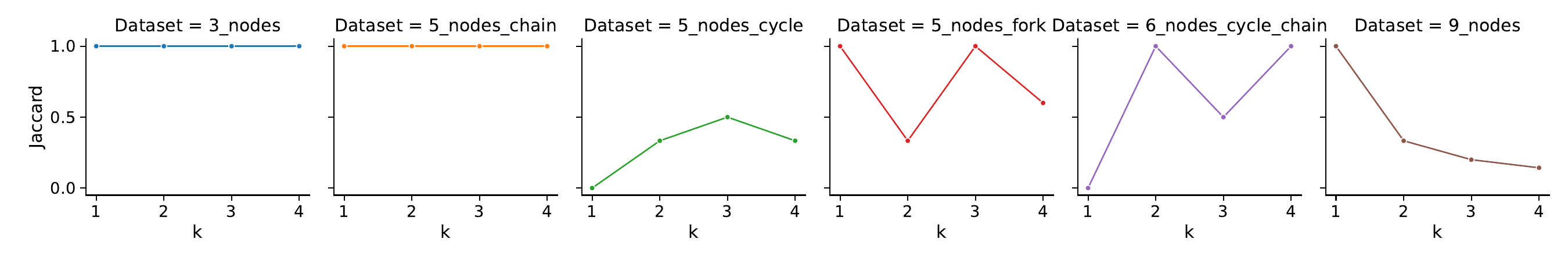}
  \caption{This figure reports the Jaccard similarity@k between the REDNT and
    the naive approach. The x-axis represents the number of states taken into
    account from the ordered lists generated by the two methods. The structures
    of the networks used for the experiments are depicted in Figures
    \ref{fig:3_nodes_experiment}, \ref{Fig:5_nodes_chain},
    \ref{Fig:5_nodes_cycle}, \ref{Fig:5_nodes_fork},
    \ref{fig:6_nodes_experiment}, and \ref{Fig:9_nodes}.}
\label{fig:jaccard}
\end{figure}

\paragraph*{ESS Data Set}

The last experiment is performed on a real data set provided by the European
Spallation Source ERIC (ESS) as described in Section  \ref{appendix:Dataset}.
The data set consists of observations of 138 alarm processes from January 2020
to March 2023. No structure was provided, thus we use the score-based structure
learning algorithm presented in \cite{nodelman2007continuous}. We chose not to
use the constraint-based algorithm \cite{BREGOLI2021105} because, in the case of
binary variables, it has been shown to be outperformed by the score-based
algorithm. The score-based algorithm penalizes the size of the parent sets,
leading to sparsity in the graphical structure. For this data, the learned graph
is composed of disconnected components. We only present the results of applying
the REDNT method for one of them. The most likely sentry state has the alarm
\textit{SpeedHighFault} set to \emph{on} and everything else set to \emph{off}
(see Figure \ref{fig:ESS_SS}). We observe that the connected component which
contains \textit{SpeedHighFault} is a rooted directed tree, and that
\textit{SpeedHighFault} is the root. This means that an alarm in
\textit{SpeedHighFault} propagates along the directed paths in the tree. A CTBN
assumes that at most one event occurs at any point in time. This is reasonable
in this application because of the high sampling rate.

The four \textit{PressureRatioHighFault} alarms in Figure \ref{fig:ESS_SS} could
be verified as consequences of the root cause \textit{SpeedHighFault}, both from
documentation and by an experienced operator. On the other hand, the alarms
\textit{StateIntervened}, \textit{CabinetFault}, and \textit{StateFault} were
not evident from the documentation and were not expected to be related to the
root alarm \textit{SpeedHighFault}. If these connections are real, this
information is relevant to operators as they may look for reasons for this
connection and enhance their process understanding. Moreover, the identified
root alarm \textit{SpeedHighFault} can be given a high priority to ensure that
the operators will be made aware of a potential cascade starting from this
alarm.

\begin{figure}[t]
  \centering
  \includegraphics[width=\textwidth]{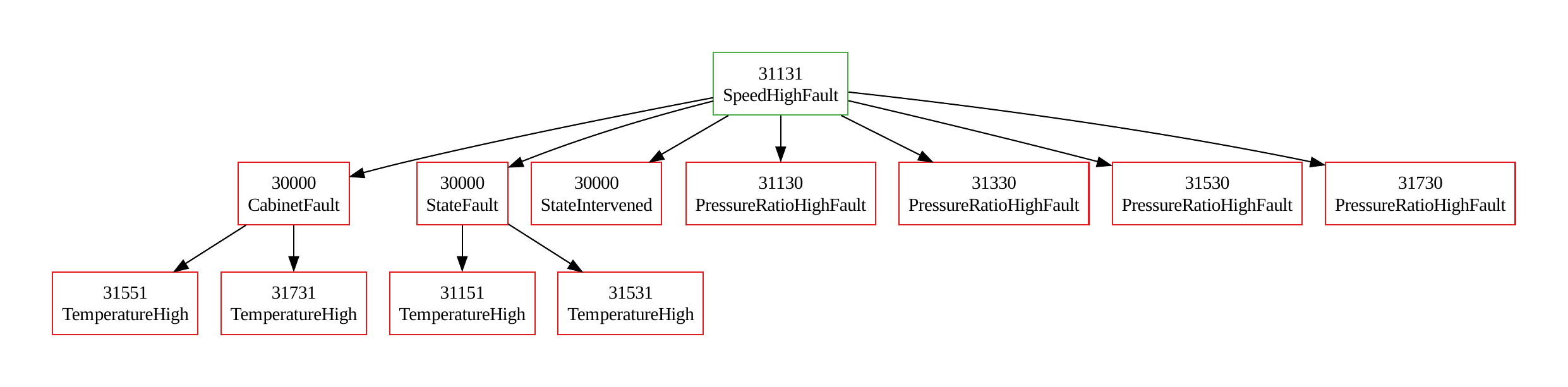}
  \caption{Graph $\G$ of the CTBN learned from the ESS data set. Figure
    \ref{fig:chain3sentry} explains the node colors.}
\label{fig:ESS_SS}
\end{figure}

The CTBN was learned with no structural input, but using a prior distribution
one can provide engineering knowledge to a Bayesian learning method. We believe
that the graphical output can easily be shown to and interpreted by engineers;
however, user studies are needed to validate this. Results like Propositions
\ref{prop:indep} and \ref{prop:indepPart} allow a formal interpretation and help
us find those components of the system that function independently when
accounting for other parts of the machine. Moreover, the sentry states that are
identified can be presented to the operators. In the example shown in figure
\ref{fig:ESS_SS} the sentry state is covered by an already existing alarm, but
if a more complex sentry state would have been identified then an additional
alarm would be needed. Using domain expertise, one can assign appropriate
instructions to sentry states and these can be presented to operators when the
machine reaches a fault state. When we have learned a CTBN from data, we may
also compute the risk of reaching each of the sentry states from the current
state. This can be provided to operators in real time to facilitate early
mitigation. This should be studied in detail in future work.

\begin{exmp}[Simplified ESS alarm network]
  We now return to our running example. This is an example system and it does
  not correspond to the learned network. Imagine that we start from the state
  with all alarms \emph{off}. Assume we have learned from data, using the CTBN
  framework, that the $\text{P1}$ alarm is very likely to go \emph{on} if both
  $\text{S4}$ and $\text{T1}$ alarms are \emph{on}. Furthermore, assume that
  $\text{P2}$ is very likely to go \emph{on} when $\text{P1}$ is \emph{on} and
  that $\text{P3}$ is very likely to go \emph{on} when $\text{P2}$ is \emph{on}.
  In this case, the state such that $\text{S4} = 1$, $\text{T1} = 1$, and such
  that every other node is zero is likely to be a sentry state. This knowledge
  can be useful in two ways. First, this can be presented to experts so that
  they can map common cascades, and their sentry state starting points, to
  underlying causes using domain knowledge. Second, during operation a warning
  can be issued when reaching a sentry state, along with the recommended course
  of action. It is also possible to compute, given the current state, the risk
  of reaching each of the sentry states within some time interval to facilitate
  an earlier warning if the expected time from reaching the sentry state to the
  actual cascade does not suffice for mitigation.
\label{exmp:system3}
\end{exmp}

\section{Discussion}
\label{sec:discussion}

In this paper we defined the concept of sentry states in CTBNs and we presented
a naive approach and a heuristic (REDNT) for identifying such sentry states. The
synthetic experiments showed that REDNT can identify the configuration of the
network from which a fast sequence of events starts. A key limitation is the
fact that the REDNT heuristic is computed for each state and the number of
states is exponential in the number of nodes. However, the simplicity of its
implementation and the effectiveness showed in the synthetic experiments make
the REDNT heuristic attractive. Moreover, only states with few active alarms may
be of interest and this reduces the computational cost. The proposed heuristic
assigns a score to each state in the state space of a CTBN; a possible extension
of this work is the identification of the contribution of each process to the
REDNT.

This paper laid the theoretical groundwork for the implementation of online
early warning systems based on the identification of sentry states. In practical
implementations, a list of sentry states can be provided to domain experts for
them to formulate appropriate actions in order to mitigate alarm cascades. This
is left for future work. Moreover, the graph representing a learned CTBN
indicates how the behavior of each alarm process depends on the states of the
other alarm processes. As illustrated in this paper, this graph also represents
conditional independences in the system. In future work, we hope to demonstrate
that the intended end users, engineers and system operators, also find this
graphical tool useful.

\subsection{Acknowledgments}

The authors would like to thank Per Nilsson for sharing his knowledge about the
cryogenics plant and for providing valuable feedback on the work presented in
this paper.


\appendix

\setcounter{table}{0}
\renewcommand{\thetable}{A\arabic{table}}
\setcounter{figure}{0}
\renewcommand{\thefigure}{A\arabic{figure}}
\section{Graphical information}
\label{app:graph}

The following proposition follows from Proposition 4 in
\cite{didelez2007graphical} and Theorem 2 in \cite{schweder1970composable}.

\begin{prop}
  If $A$ is ancestral, then the subprocess $\X_A = (\X_i)_{i\in A}$ is a CTBN
  with transition matrices $Q_{\X_i\vert \pa(\X_i)}$ and graph $\G_A$.
\end{prop}

\begin{proof}[Proof of Proposition \ref{prop:indepPart}]
  Note that $A$, $B$, and $C$ must be disjoint. We consider a connecting path
  between $A$ and $B$  in $(\G_{\an(A\cup B\cup C)})^m$ which does not intersect
  $C$,
  \begin{align*}
    \X_{i_0} - \X_{i_1} - \ldots - \X_{i_m}.
  \end{align*}
  Let $f: \X \rightarrow D$ be the unique map such that if $f(\X_i) = D_l$, then
  $\X_i\in D_l$. We consider the walk
  \begin{equation*}
    f(\X_{i_0}) - f(\X_{i_1}) - \ldots - f(\X_{i_m}).
  \end{equation*}
  and argue that this walk, or a subwalk, is present in
  $(\mathcal{D}_{\an(\underline{A}\cup\underline{B}\cup\underline{C})})^m$ and
  is not intersected by $\underline{C}$. Every node on the original walk is in
  $\an(A\cup B\cup C)$ in $\G$, so every node on the above walk is in
  $\an(\underline{A}\cup\underline{B}\cup\underline{C})$ in $\mathcal{D}$. We
  remove nodes such that no adjacent nodes are equal (note that the result is a
  nontrivial walk). If an edge on the original walk corresponds to a directed
  edge in $\G$, then it is also in $\mathcal{D}$. Assume it does not correspond
  to a directed edge on the original walk. It then corresponds to a `moral'
  edge, $\X_{i_j} \rightarrow \X_k \leftarrow \X_{i_{j+1}}$ in $\G$, and these
  must be in different $D_i$. In this case, $\X_{i_j} - \X_{i_{j+1}}$ is also in
  $(\mathcal{D}_{\an(\underline{A}\cup\underline{B}\cup\underline{C})})^m$. No
  node can be in $\underline{C}$ on this walk. We can reduce this to a path such
  that no node is repeated. Note that the end nodes are in $\underline{A}$ and
  $\underline{B}$, respectively.
\end{proof}

\begin{prop}
  Let ${D}_1, \ldots, {D}_m$ be the strongly connect components of $\G$. If
  there are no edges between ${D}_i$ and ${D}_j$, $i\neq j$, in $\G$, then
  $\overline{\X}(t)_{D_i} \indep \overline{\X}(t)_{D_j} \mid
  \overline{\X}(t)_{p_i}$ or $\overline{\X}(t)_{D_i} \indep
  \overline{\X}(t)_{D_j} \mid \overline{\X}(t)_{p_j}$ where $p_i = \bigcup_{D_k
  \in \pa(D_i)} D_k$ and $p_j = \bigcup_{D_k \in \pa(D_j)} D_k$.
\end{prop}

\begin{proof}
  If there are no edges between any node in ${D}_i$ and any node in ${D}_j$,
  then ${D}_i$ and ${D}_j$ are not adjacent in the condensation of $\G$. The
  condensation is acyclic, so we can without loss of generality assume that
  ${D}_i$ is not a descendant of ${D}_j$. There are no descendants of ${D}_j$ in
  $\an(\{{D}_i , {D}_j\} \cup \pa({D}_j))$ and this means that ${D}_i$ and
  ${D}_j$ are separated by $\pa({D}_j)$ in $(\mathcal{D}_{\an(\{{D}_i,{D}_j\}
  \cup \pa({D}_j))})^m$ where $\mathcal{D}$ is the condensation of $\G$. The
  result follows from Propositions \ref{prop:indepPart} and \ref{prop:indep}.
\end{proof}

\section{Cascade Identification}
\label{subsection:cascade_identification}

Informally, a cascade of events is a fast sequence of transitions; where fast is
relative to the rest of the transitions that are observable during the evolution
of the process. Starting from this informal definition, we can develop a naive
approach to identifying such cascades in a trajectory. First of all we need to
identify two quantities: - $\lambda_{ft}$: the \textit{fast threshold}
determines when two consecutive transitions are considered to occur
\textit{fast}. -  $\lambda_{mcl}$: the \textit{minimum cascade length}
determines the minimum number of fast consecutive events to be considered a
cascade.

Given the two parameters  the identification procedure consists of iterating
over the entire trajectory and identifying subsets of consecutive transitions
with length at least $\lambda_{mcl}$ and with a transition time between each
pair of consecutive events of less than $\lambda_{ft}$. This approach can also
be used to identify a \textit{sentry state}. Indeed, once a cascade of events
has been identified, the sentry state is the state from which the cascade
begins.

The main limitation of this approach is the difficulty of identifying the
correct parameters as it requires knowing in advance common durations and sizes
of event cascades.

In addition, we define two simple quantities: \textit{Naive Count} - the number
of times a state starts a cascade, and \textit{Naive Score} - the fraction of
times that observing a specific state coincides with the start of a cascade.

\section{Synthetic Experiments}
\label{app:synth}

\begin{table}[H]
\centering
\begin{tabular}{cccc|c|c|cc|c|c|c|cc|}
\cline{1-3} \cline{5-8} \cline{10-13}
\multicolumn{1}{|c|}{A} & 0    & \multicolumn{1}{c|}{1}    &  & A                  & B & 0    & 1     &  & B                  & C & 0     & 1     \\ \cline{1-3} \cline{5-8} \cline{10-13}
\multicolumn{1}{|c|}{0} & -1.0 & \multicolumn{1}{c|}{1.0}  &  & \multirow{2}{*}{0} & 0 & -0.1 & 0.1   &  & \multirow{2}{*}{0} & 0 & -0.1  & 0.1   \\
\multicolumn{1}{|c|}{1} & 5.0  & \multicolumn{1}{c|}{-5.0} &  &                    & 1 & 15.0 & -15.0 &  &                    & 1 & 15.0  & -15.0 \\ \cline{1-3} \cline{5-8} \cline{10-13}
                        &      &                           &  & \multirow{2}{*}{1} & 0 & 15.0 & -15.0 &  & \multirow{2}{*}{1} & 0 & -15.0 & 15.0  \\
                        &      &                           &  &                    & 1 & 0.1  & -0.1  &  &                    & 1 & 0.1   & -0.1  \\ \cline{5-8} \cline{10-13}
\end{tabular}
  \caption{Conditional Intensity Matrices used for the example in Figure
    \ref{fig:3_nodes_experiment}. Process A has no parents and therefore its
    transition rate only depends on its own state: If A is in state 0
    (\emph{off}), then its transition rate (to state 1 (\emph{on})) is $1.0$.
    Process B has a single parent, process A. The states of processes A and B
    determine the transition rate of process B. If A is in state 0 (\emph{off})
    and B is in state 0 (\emph{off}), then B transitions to state 1 (\emph{on})
    with rate $0.1$. A CTBN is defined from its CIMs and its initial
    distribution. Its graph illustrates the dependence structure in the CIMs.}
\label{table:3_nodes_cims}
\end{table}

\begin{table}[H]
\centering
\begin{tabular}{|ccc|cccc|}
\hline
A & B & C & EDNT  & REDNT & Naive Score & Naive Count\\ \hline
\textbf{1} & \textbf{0} & \textbf{0} & \textbf{6.316} & \textbf{1.589} & \textbf{0.35}  & \textbf{304}\\
1 & 0 & 1 & 6.444 & 1.359 & 0.21 & 13\\
\textbf{0} & \textbf{1} & \textbf{0} & \textbf{5.394} & \textbf{1.357} & \textbf{0.16} & \textbf{41}\\
\textbf{0} & \textbf{0} & \textbf{1} & \textbf{4.740} & \textbf{1.192} & \textbf{0.03} & \textbf{26}\\
0 & 1 & 1 & 5.511 & 1.163 & 0.22 & 153\\
1 & 1 & 0 & 6.173 & 1.145 & 0.08  & 57\\
1 & 1 & 1 & 5.455 & 1.0   & 0.03 & 19\\
\textbf{0} & \textbf{0} & \textbf{0} & \textbf{3.976} & \textbf{1.0}   & \textbf{0.02} & \textbf{24}\\ \hline
\end{tabular}
  \caption{Values of EDNT, REDNT, Naive Score, and Naive Count for the CTBN
    depicted in Figure \ref{fig:3_nodes_experiment}. Higher values of REDNT
    indicate CTBN states that are more likely to be sentry states. One should
    note that high-scoring states with few alarms (bold rows) are more
    interesting in our application as they correspond to states that occur
    before strong cascading behavior.}
\label{table:3_nodes_results}
\end{table}

\begin{table}[H]
\centering
\begin{tabular}{|cccccc|cccc|}
\hline
A       & B       & C       & D       & E       & F       & EDNT    & REDNT   & Naive Score & Naive Count\\ \hline
0 & 0 & 1 & 0 & 0 & 0 & 12.98 & 1.46 & 0.24 & 2172 \\
0 & 0 & 0 & 1 & 0 & 0 & 11.33 & 1.28 & 0.25 & 2156\\
0 & 1 & 0 & 0 & 0 & 0 & 10.76 & 1.21 & 0.04 & 848\\
0 & 0 & 0 & 0 & 1 & 0 & 10.75 & 1.21 & 0.14 & 1533\\
1 & 0 & 0 & 0 & 0 & 0 & 10.24 & 1.15 & 0.02 & 341\\
0 & 0 & 0 & 0 & 0 & 1 & 9.90 & 1.12 &  0.03 & 651\\
0 & 0 & 0 & 0 & 0 & 0 & 8.87 & 1.0 & 0.01 & 426\\\hline
\end{tabular}
  \caption{Values of EDNT, REDNT, Naive Score, and Naive Count  of the states
    with at most one active alarm for the CTBN depicted in Figure
    \ref{fig:6_nodes_experiment}.}
\label{table:6_nodes_results}
\end{table}

\begin{figure}[H]
  \centering
  \includegraphics[width=0.8\textwidth]{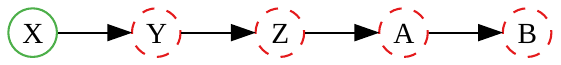}
  \caption{Graph $\G$ of a CTBN model consisting of a chain of five processes.}
\label{Fig:5_nodes_chain}
\end{figure}

\begin{figure}[H]
  \centering
  \includegraphics[width=0.8\textwidth]{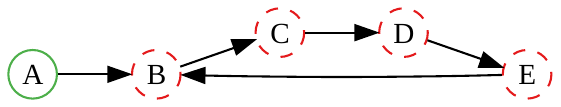}
  \caption{Graph $G$ of a CTBN model consisting of five processes, including a
    cycle.}
\label{Fig:5_nodes_cycle}
\end{figure}

\begin{figure}[H]
  \centering
  \includegraphics[width=0.5\textwidth]{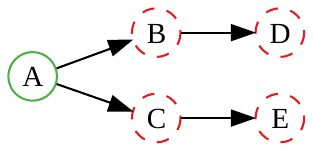}
  \caption{Graph $\G$ of a CTBN model consisting of five processes. The graph
    $\G$ contains a bifurcation after the root node $A$.}
\label{Fig:5_nodes_fork}
\end{figure}

\begin{figure}[H]
  \centering
  \includegraphics[width=0.5\textwidth]{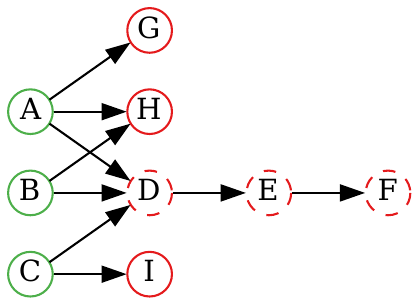}
  \caption{Graph $\G$ of a CTBN model consisting of nine processes and with a
    more complex structure. The sentry state has three active alarms.}
\label{Fig:9_nodes}
\end{figure}

\end{document}